\documentclass[10pt,twocolumn,letterpaper]{article}

\usepackage{cvpr}
\usepackage{times}
\usepackage{epsfig}
\usepackage{graphicx}
\usepackage{amsmath}
\usepackage{amssymb}

\usepackage{color}
\usepackage[linesnumbered,lined,boxed,commentsnumbered,ruled,vlined]{algorithm2e}
\definecolor{MyAlgCommentColor}{RGB}{31,91,155}
\def\MyAlgComment#1{{\hspace{5pt}\color{MyAlgCommentColor}$\triangleright$ \it #1}}
\let\oldnl\nl
\newcommand{\nonl}{\renewcommand{\nl}{\let\nl\oldnl}}\SetKwFor{Function}{function}{}{}

\newcommand{\comment}[1]{}

\newcommand{\argmax}{\operatornamewithlimits{argmax}}

\long\def\ignore#1{}
\def\myparagraph#1{\vspace{4pt}\noindent{\bf #1~~}}

\def\calN{{\cal N}}
\def\calP{{\cal P}}

\newcommand{\bR}{\mathbf{R}}

\newcommand{\bw}{\mathbf{w}}

\newcommand{\bx}{\mathbf{x}}
\newcommand{\by}{\mathbf{y}}
\newcommand{\sP}{\left|\cal{P}\right|}
\newcommand{\sN}{\left|\cal{N}\right|}

\def\Call#1#2{{{\tt #1}{(#2)}}}

\newtheorem{theorem}{Theorem}

\newtheorem{lemma}[theorem]{Lemma}

\newtheorem{proposition}[theorem]{Proposition}
\newtheorem{remark}{Remark}
\newtheorem{obs}{Observation}

\newenvironment{proof}[1][Proof] {\noindent {\em #1. }}{\hfill$\square$}

\usepackage[pagebackref=true,breaklinks=true,letterpaper=true,colorlinks,bookmarks=false]{hyperref}

\cvprfinalcopy 

\def\cvprPaperID{3595} 

\ifcvprfinal\fi
\begin{document}

\title{Efficient Optimization for Rank-based Loss Functions}

\author{
\begin{tabular}[t]{c@{\extracolsep{120pt}}c} 
Pritish Mohapatra\thanks{The first two authors contributed equally and can be reached at \emph{pritish.mohapatra@research.iiit.ac.in}, \emph{michal.rolinek@tuebingen.mpg.de} respectively.}  & Michal Rol\'inek\footnotemark[1]  \\ 
        IIIT Hyderabad & MPI T\"ubingen \\ \\
\end{tabular}
\\
\begin{tabular}[t]{c@{\extracolsep{40pt}}c@{\extracolsep{40pt}}c} 
C.V. Jawahar  & Vladimir Kolmogorov & M. Pawan Kumar \\
        IIIT Hyderabad & IST Austria & University of Oxford, Alan Turing Institute
\end{tabular}
}



\maketitle

\begin{abstract}
   The accuracy of information retrieval systems is often measured using complex loss functions such as the average precision ({\sc ap}) or the normalized discounted cumulative gain ({\sc ndcg}). Given a set of positive and negative samples, the parameters of a retrieval system can be estimated by minimizing these loss functions. However, the non-differentiability and non-decomposability of these loss functions does not allow for simple gradient based optimization algorithms. This issue is generally circumvented by either optimizing a structured hinge-loss upper bound to the loss function or by using asymptotic methods like the direct-loss minimization framework. Yet, the high computational complexity of loss-augmented inference, which is necessary for both the frameworks, prohibits its use in large training data sets. To alleviate this deficiency, we present a novel quicksort flavored algorithm for a large class of non-decomposable loss functions. We provide a complete characterization of the loss functions that are amenable to our algorithm, and show that it includes both {\sc ap} and {\sc ndcg} based loss functions. Furthermore, we prove that no comparison based algorithm can improve upon the computational complexity of our approach asymptotically. We demonstrate the effectiveness of our approach in the context of optimizing the structured hinge loss upper bound of {\sc ap} and {\sc ndcg} loss for learning models for a variety of vision tasks. We show that our approach provides significantly better results than simpler decomposable loss functions, while requiring a comparable training time.
\end{abstract}
\vspace*{-3mm}
\section{Introduction}

Information retrieval systems require us to rank a set of samples according to their relevance to a query. The risk of the predicted ranking is measured by a user-specified loss function. Several intuitive loss functions have been proposed in the literature. These include simple decomposable losses (that is, loss functions
that decompose over each training sample) such as 0-1 loss~\cite{lin2002support,morik1999combining} and the area under the ROC curve~\cite{bartell1994automatic,herschtal2004optimising}, as well as
the more complex non-decomposable losses (that is, loss functions that depend on the entire training data set)
such as the average precision ({\sc ap})~\cite{caruana2004ensemble,yuesigir07} and the normalized discounted cumulative gain ({\sc ndcg})~\cite{chakrabartikdd08}.

When learning a retrieval system, one can use a training objective that is agnostic to the risk, such as in the case of Lambda{\sc mart}~\cite{burges2007learning_LambdaRank}. In this work, we focus on approaches that explicitly take into account the loss function used to measure the risk. Such approaches can use any one of the many machine learning frameworks such as structured support vector machines ({\sc ssvm})~\cite{taskarnips03,tsochantaridisicml04},
deep neural networks~\cite{szegedynips13}, decision forests~\cite{kimmoml06}, or boosting~\cite{shenarxiv10}. To estimate the parameters of the framework, they employ a training objective that is closely related to the empirical risk computed over a large training data set. Specifically, it is common practice to employ either a structured hinge upper bound to the loss function \cite{chakrabartikdd08,yuesigir07}, or an asymptotic alternative such as direct loss minimization \cite{hazan2010direct_directLossStruct,song2016training_directLossDeep}.

The feasibility of both the structured hinge loss and the direct loss minimization approach depends on the computational efficiency of the {\em loss-augmented inference} procedure. When the loss function is decomposable, the loss-augmented inference problem can be solved efficiently by independently considering each training sample. However, for non-decomposable loss functions, it presents a
hard computational challenge. For example, given a training data set with $P$ positive (relevant to the query) and $N$ negative (not relevant to the query) samples, the best known algorithms for loss-augmented inference for {\sc ap} and {\sc ndcg} loss functions have a complexity of $O(PN+N\log N)$~\cite{chakrabartikdd08,yuesigir07}. Since the number of negative samples $N$
can be very large in practice, this prohibits their use on large data sets.

In order to address the computational challenge of non-decomposable loss functions such as those based on {\sc ap} and {\sc ndcg}, we make three
contributions. First, we characterize a large class of ranking based loss functions that are amenable to
a novel quicksort flavored optimization algorithm for the corresponding loss-augmented inference problem.
We refer to the class of loss functions as {\em {\sc qs}-suitable}.
Second, we show that the {\sc ap} and the {\sc ndcg} loss functions are {\sc qs}-suitable, which allows us to reduce the
complexity of the corresponding loss-augmented inference to $O(N\log P)$. Third, we prove that there
cannot exist a comparison based method for loss-augmented inference that can provide a better asymptotic
complexity than our quicksort flavored approach. It is worth noting that our work is complementary to previous algorithms that have been proposed for the {\sc ap} based loss functions~\cite{mohapatranips14}. Specifically, while the method of~\cite{mohapatranips14} cannot improve the asymptotic complexity of our loss-augmented inference algorithm, it can be used to reduce the runtime of a subroutine.

For the sake of clarity, we limit our discussion to the structured hinge loss upper bound to the loss function. However, as our main contribution is to speed-up loss-augmented inference, it is equally applicable to direct loss minimization. We demonstrate the efficacy of our approach on the challenging problems of action recognition, object detection and image classification, using publicly available data sets. Rather surprisingly, we show that in case of some models, parameter learning by optimizing complex non-decomposable {\sc ap} and {\sc ndcg} loss functions can be carried out faster than by optimizing simple decomposable 0-1 loss. Specifically, while each loss-augmented inference call is more expensive for
{\sc ap} and {\sc ndcg} loss functions, it can take fewer calls in practice to estimate the parameters of the corresponding model.

\section{Background}
\label{sec:backg}

We begin by providing a brief description of a general retrieval framework that employs a rank-based loss function, hereby referred to as the
ranking framework. Note that this framework is the same as or generalizes the ones employed in previous works~\cite{chakrabartikdd08,joachimsicml05,mohapatranips14,song2016training_directLossDeep,yuesigir07}.
The two specific instantiations of the ranking framework that are of interest to us employ the average precision ({\sc ap}) loss and the normalized
discounted cumulative gain ({\sc ndcg}) loss respectively. A detailed description of the two aforementioned loss functions is provided in the subsequent subsection.

\subsection{The Ranking Framework}

\myparagraph{Input.} The input to this framework is a set of $n$ samples, which we denote by
${\bf X} = \{{\bf x}_i,i=1,\dots,n\}$. For example, each sample can represent an image and a bounding box of a person present
in the image. In addition, we are also provided with a query, which in our example could represent an action such as `jumping'.
Each sample can either belong to the positive class (that is, the sample is relevant to the query) or the negative class (that is,
the sample is not relevant to the query). For example, if the query represents the action `jumping' then a sample is positive if
the corresponding person is performing the jumping action, and negative otherwise. The set of positive and the negative samples are denoted by $\cal{P}$ and $\cal{N}$ respectively. which we assume are provided during training, but are not known during testing.

\myparagraph{Output.}
Given a query and a set of $n$ samples {\bf X}, the desired output of the framework is a ranking of the samples according to their relevance to the query. This is often represented by a ranking matrix ${\bf R} \in \{-1,0,1\} ^ {n\times n}$ such that $\bR_{{\bf x},{\bf y}}=$ 1 if ${\bf x}$ is ranked higher than ${\bf y}$, -1 if ${\bf x}$ is ranked lower than ${\bf y}$ and 0 if ${\bf x}$ and ${\bf y}$ are ranked the same.
In other words, the matrix ${\bf R}$ is an anti-symmetric that represents the relative ranking of a pair of samples.

Given the sets $\cal{P}$ and $\cal{N}$ during training, we construct a ground truth ranking matrix ${\bf R}^*$, which ranks each positive sample above all
the negative samples. Formally, the ground truth ranking matrix ${\bf R}^*$ is defined such that ${\bf R}^*_{{\bf x},{\bf y}} = $ 1 if ${\bf x} \in \cal{P}$ and ${\bf y} \in {\cal N}$, -1 if ${\bf x} \in \cal{N}$ and ${\bf y} \in {\cal P}$, and 0 if ${\bf x},{\bf y} \in \cal{P}$ or ${\bf x},{\bf y} \in {\cal N}$. Note that the ground truth ranking matrix only defines a partial ordering on the samples since ${\bf R}^*_{i,j} = 0$ for all pairs of positive and negative
samples.
We will refer to rankings where no two samples are ranked equally as {\em proper rankings}. Without loss of generality, we will treat all rankings other than
the ground truth one as a proper ranking by breaking ties arbitrarily.

\myparagraph{Discriminant Function.}
Given an input set of samples {\bf X}, the discriminant function $F({\bf X},{\bf R};{\bf w})$ provides a score for any candidate ranking ${\bf R}$. Here, the term {\bf w} refers to the parameters of the discriminant function. We assume that the discriminant function is piecewise differentiable with respect to its parameters {\bf w}. One popular example of the discriminant function used throughout the ranking literature is the following:
\begin{equation}
F({\bf X},{\bf R};{\bf w}) = \frac{1}{\left|\cal P\right| \left| N \right|} \sum_{\bx \in \cal{P}} \sum_{\by \in \cal{N}} {\bf R}_{\bx,\by}(\phi(\bx; {\bf w}) - \phi(\by; {\bf w})).
\label{eq:jointFeature}
\end{equation}
Here, $\phi(\bx; {\bf w})$ is the score of an individual sample, which can be provided by a structured {\sc svm} or a deep neural network with parameters {\bf w}.

\myparagraph{Prediction.}
Given a discriminant function $F({\bf X}, {\bf R}; {\bf w})$ with parameters {\bf w}, the ranking of an input set of samples ${\bf X}$
is predicted by maximizing the score, that is, by solving the following optimization problem:
\begin{equation}
{\bf R}({\bf w}) = \argmax_{\bf R} F({\bf X},{\bf R};{\bf w}).
\label{eq:rankPrediction}
\end{equation}
The special form of the discriminant function in equation~(\ref{eq:jointFeature})
enables us to efficiently obtain the predicted ranking ${\bf R}({\bf w})$ by sorting the samples in descending order of their individual scores $\phi(\bx; {\bf w})$. We refer the reader to~\cite{joachimsicml05,yuesigir07} for details.

\myparagraph{Parameter Estimation.} We now turn towards estimating the parameters of our model given input samples ${\bf X}$, together
with their classification into positive and negative sets $\cal{P}$ and $\cal{N}$ respectively. To this end, we minimize the risk of prediction computed using a user-specified loss function $\Delta({\bf R}^*,{\bf R}({\bf w}))$, where ${\bf R}^*$ is the
ground truth ranking that is determined by $\cal{P}$ and $\cal{N}$ and ${\bf R}({\bf w})$ is the predicted ranking as shown in
equation~(\ref{eq:rankPrediction}). We estimate the parameters of our model as\vspace*{-2.5mm}
\begin{align}
\label{eq:empRiskMin}
{\bw}^* = \min_{\bw} \mathbb{E}[\Delta({\bf R}^*,{\bf R}({\bf w}))].
\end{align}
In the above equation, the expectation is taken with respect to the data distribution.

\myparagraph{Optimization for Parameter Estimation.} For many intuitive rank based loss functions such as {\sc ap} loss and {\sc ndcg} loss, owing to their non-differentiability and non-decomposability, problem (\ref{eq:empRiskMin}) can be difficult to solve using simple gradient based methods. One popular approach is to modify problem (\ref{eq:empRiskMin}) to instead minimize a structured hinge loss upper bound to the user-specified loss. We refer the reader to \cite{yuesigir07} for further details about this approach.

Formally, the model parameters can now be obtained by solving the following problem:\vspace*{-1mm}
\begin{align}
\label{eq:hingeUB_obj}
&~~~~~~~~~~~~~~~~~~~~~~~~~~~~~~~{\bw}^* = \min_{\bw} \mathbb{E}[J(\bw)] \\
&J(\bw) = \max_{\bf R} \Delta(\bR^*, \bR) + F({\bf X},{\bf R};{\bf w}) - F({\bf X},{\bf R}^*;{\bf w}) \nonumber
\end{align}
The function $J(\bw)$ in problem (\ref{eq:hingeUB_obj}) is continuous and piecewise differentiable, and is amenable to gradient based optimization. The semi-gradient \footnote{For a continuous function $f(x)$ defined on a domain of any generic dimension, we can define semi-gradient $\nabla_s f(x)$ to be a random picking from the set $\{\nabla f(t): ||x-t||<\epsilon\}$, for a sufficiently small $\epsilon$.} of $J(\bw)$ takes the following form: 
\begin{align}
\label{eq:hingeUB_grad}
&\nabla_{\bw} J(\bw) = \nabla_{\bw} F({\bf X}, \bar{\bR}; {\bw}) - \nabla_{\bw} F({\bf X}, {\bR}^*; {\bw}), \\
&\textrm{with},~~~\bar{\bR} = \argmax_{\bR} \Delta(\bR^*, \bR) + F({\bf X}, \bR; {\bw}). \label{eq:LAI}
\end{align}
Borrowing terminology from the structured prediction literature \cite{joachimsml09, yuesigir07}, we call $\bar{\bR}$ the {\em most violating ranking} and problem (\ref{eq:LAI}) as the {\em loss-augmented inference} problem. An efficient procedure for loss-augmented inference is key to solving problem (\ref{eq:hingeUB_obj}).

While we focus on using loss-augmented inference for estimating the semi-gradient, it can also be used as the cutting plane~\cite{joachimsml09} and the conditional gradient of the dual of problem (\ref{eq:hingeUB_obj}). In addition to this, loss-augmented inference is also required for solving problem (\ref{eq:empRiskMin}) using the direct loss minimization framework \cite{song2016training_directLossDeep}.

\subsection{Loss Functions}

While solving problem (\ref{eq:LAI}) is non-trivial, especially for non-decomposable loss functions, the method we propose in this paper allows for an efficient loss-augmented inference procedure for such complex loss functions. For our discussion, we focus on two specific non-decomposable loss functions. The first is the average precision ({\sc ap}) loss, which
is very popular in the computer vision community as evidenced by its use in the various challenges of {\sc pascal voc}~\cite{everinghamijcv10}.
The second is the normalized discounted cumulative gain ({\sc ndcg}) loss, which is very popular in the information
retrieval community~\cite{chakrabartikdd08}.

\myparagraph{Notation.} In order to specify the loss functions, and our efficient algorithms for problem~(\ref{eq:hingeUB_obj}), it would be
helpful to introduce some additional notation. We define $ind(\bx)$ to be the index of a sample $\bx$ according to the ranking $\bR$. Note that the notation does not explicitly depend on $\bR$ as the ranking will always be clear from context.
If $\bx \in \cal{P}$ (that is, for a positive
sample), we define $ind^+(\bx)$ as the index of $\bx$ in the total order of positive samples induced by $\bR$. For example, if $\bx$ is the highest ranked
positive sample then $ind^+(\bx) = 1$ even though $ind(\bx)$ need not necessarily be 1 (in the case where some negative samples are ranked higher than $\bx$).
For a negative sample $\bx \in \cal{N}$, we define
$ind^-(\bx)$ analogously: $ind^-(\bx)$ is the index of $\bx$ in the total order of negative samples induced by $\bR$.

\myparagraph{AP Loss.} Using the above notation, we can now concisely define the average precision ({\sc  ap}) loss of a proper ranking $\bR$ given
the ground truth ranking $\bR^*$ as follows:\vspace*{-1mm}
$$\Delta_{AP}(\bR^*, \bR) = 1 - \frac{1}{\sP} \sum_{\bx \in {\cal P}} \frac{ind^+(\bx)}{ind(\bx)}.$$\vspace*{-1mm}
For example, consider an input ${\bf X} = \{\bx_1,\cdots,\bx_8\}$ where $\bx_i \in \cal{P}$ for $1 \leq i \leq 4$, and
$\bx_i \in \cal{N}$ for $5 \leq i \leq 8$, that is, the first 4 samples are positive while the last 4 samples
are negative. If the proper ranking $\bR$ induces the order\vspace*{-1mm}
\begin{equation}
(\bx_1,\bx_3,\bx_8,\bx_4,\bx_5,\bx_2,\bx_6,\bx_7),
\label{eq:exampleOrder}
\end{equation}\vspace*{-3mm}
$$\text{then},~~\Delta_{AP}(\bR^*, \bR) = 1 - \frac{1}{4} \left(\frac{1}{1} + \frac{2}{2} + \frac{3}{4} + \frac{4}{6} \right) \approx 0.146.$$

\myparagraph{NDCG Loss.} We define a discount \footnote{Chakrabarti {\em et al. }\cite{chakrabartikdd08} use a slightly modified definition of the discount $D(\cdot)$. For a detailed discussion about it the reader can refer to the Appendix (Supplementary).} $D(i) = 1/\log_2(1+i)$ for all $i = 1,\cdots, |\cal{N}| + |\cal{P}|$. This allows us to obtain a loss
function based on the normalized discounted cumulative gain as \vspace*{-4mm}
$$\Delta_{NDCG}(\bR^*, \bR) = 1-\frac{\sum_{\bx \in {\cal P}} D(ind(\bx))}{\sum_{i = 1}^{\sP} D(i)}.$$
For example, consider the aforementioned input where the first four samples are positive and the last four samples are negative. For
the ranking $\bR$ that induces the order~(\ref{eq:exampleOrder}), we can compute \vspace*{-2.5mm}
\begin{eqnarray}
\Delta_{{\scriptstyle NDCG}}(\hat{\bR}, \bR) = 1 - \frac{1 + \log_2^{-1}3 + \log_2^{-1}5+\log_2^{-1}7}{1 + \log_2^{-1}3 + \log_2^{-1}4+\log_2^{-1}5} \nonumber \\
 \approx 0.056. \nonumber
\end{eqnarray}
\vspace{-7mm}

\noindent Both {\sc ap} loss and {\sc ndcg} loss are functions of the entire dataset and are not decomposable onto individual samples.

\section{Quicksort Flavored Optimization}
\label{sec:qs}

\def\LeftInd{\ensuremath{\ell^-}}
\def\RightInd{\ensuremath{r^-}}
\def\Left{{\ensuremath{\ell^+}}}
\def\Right{{\ensuremath{r^+}}}

\def\Med{{\ensuremath{m}}}


In order to estimate the parameters $\bw$ in the ranking framework by solving problem (\ref{eq:hingeUB_obj}), we need to compute the semi-gradient of $J(\bw)$. To this end, given the current estimate of parameters {\bf w}, as well as a set of samples ${\bf X}$, we are interested in obtaining the most violated ranking
by solving problem~(\ref{eq:LAI}). At first glance, the problem seems to require us to obtain a ranking matrix $\bar{\bf R}$. However, it
turns out that we do not explicitly require a ranking matrix. 

In more detail, our algorithm uses an intermediate representation of the ranking using the notion of interleaving ranks. Given a ranking ${\bf R}$ and
a negative sample $\bx$, the interleaving rank $rank(\bx)$ is defined as one plus the number of positive samples preceding $\bx$ in $\bR$. Note that,
similar to our notation for $ind(\cdot)$, $ind^+(\cdot)$ and $ind^-(\cdot)$, we have dropped the dependency of $rank(\cdot)$ on ${\bf R}$ as the ranking matrix would be
clear from context. The interleaving rank of all the samples does not specify the total ordering of all the samples according to $\bR$ as it ignores the
relative ranking of the positive samples among themselves, and the relative ranking of the negative samples among themselves. However, as will be seen shortly, for
a large class of ranking based loss functions, interleaving ranks corresponding to the most violating ranking are sufficient to compute the semi-gradient as in equation (\ref{eq:hingeUB_grad}).

In the rest of the section, we discuss the class of loss functions that are amenable to a quicksort flavored algorithm, which we call {\sc qs}-suitable loss functions. We then describe and analyze our quicksort flavored approach for finding the interleaving rank in some detail. For brevity and simplicity of exposition, in the following sub-section, we restrict our discussion to the properties of {\sc qs}-suitable loss functions that are necessary for an intuitive explanation of our algorithm. For a thorough discussion on the characterization and properties of {\sc qs}-suitable loss functions, we refer the interested reader to the full version of the paper \cite{mohapatra2016efficient}.

\subsection{QS-Suitable Loss Functions}

As discussed earlier, many popular rank-based loss functions happen to be non-decomposable. That is, they can not be additively decomposed onto individual samples. However, it turns out that a wide class of such non-decomposable loss functions can be instead additively decomposed onto the negative samples. Formally, for some functions $\delta_j \colon \{1, \dots, \sP+1\} \to \mathbb{R}$ for $j = 1, \dots, \sN$, for a proper ranking $\bR$ one can write \vspace{-2mm}
$$\Delta(\bR^*, \bR) = \sum_{\bx \in \calN} \delta_{ind^-(\bx)}(rank(\bx)).\vspace{-2mm}$$ 
We will call this the {\em negative-decomposability} property. 

Further, many of those rank-based loss functions do not depend on the relative order of positive or negative samples among themselves. Rather, the loss for a ranking $\bR$, $\Delta(\bR^*, \bR)$, depends only on the interleaving rank of positive and negative samples corresponding to $\bR$. We will call this the {\em interleaving-dependence} property.

As will be evident later in the section, the above properties in a loss function allows for an efficient quicksort flavored divide and conquer algorithm to solve the loss augmented problem. We formally define the class of loss functions that allow for such a quicksort flavored algorithm as {\em {\sc qs}-suitable} loss functions. The following proposition establishes the usefulness for such a characterization.
\vspace{-1mm}
\begin{proposition}\label{prop:suitable} Both $\Delta_{AP}$ and $\Delta_{NDCG}$ are QS-suitable.
\end{proposition}
\vspace{-1mm}
The proof of the above proposition is provided in Appendix (supplementary). Having established that both the {\sc ap} and the {\sc ndcg} loss are {\sc qs}-suitable, the rest of the section will deal with a general {\sc qs}-suitable loss function. A reader who is interested in employing another loss function need only check whether the required conditions are satisfied in order to use our approach.

\subsection{Key Observations for QS-Suitable Loss}
\label{subsec:keyObservations}

Before describing our algorithm in detail, we first provide some key observations which enable efficient optimization for {\sc qs}-suitable loss functions. To this end, let us define an array $\{s^+_i\}_{i=1}^{\sP}$ of positive sample scores and an array $\{s^-_i\}_{i=1}^{\sN}$ of negative sample scores. Furthermore, for purely notational purposes, let $\{ s^*_i\}$ be the array $\{ s^-_i\}$ sorted in descending order. For $j \in \{1, \dots \sN \}$ we denote the index of $s^-_j$ in $\{s^*_i\}$ as $j^*$.

With the above notation, we describe some key observations regarding {\sc qs}-suitable loss functions. Their proofs are for most part straightforward generalizations of results that appeared in \cite{mohapatranips14,yuesigir07} in the context of the {\sc ap} loss and can be found in Appendix (Supplementary). Using the interleaving-dependence property of {\sc qs}-suitable loss functions and structure of the discriminant function as defined in equation (\ref{eq:jointFeature}), we can make the following observation.
\vspace{-1mm}
\begin{obs}\label{obs:sorted} An optimal solution $\bar{\bR}$ of problem~(\ref{eq:LAI}) would have positive samples appearing in the descending order of their scores $s^+_i$ and also the negative samples appearing in descending order of their scores $s^-_i$.
\end{obs}
\vspace{-1mm}
Now, in order to find the optimal ranking $\bar{\bR}$, it would seem natural to sort the arrays $\{s^+_i\}$ and $\{s^-_i\}$ in descending order and then find the optimal interleaving ranks $rank({\bf x})$ for all ${\bf x}\in \calN$. However, we are aiming for complexity below $O(\sN \log \sN)$, therefore we can not afford to sort the negative scores. On the other hand, since $\sP << \sN$, we are allowed to sort the array of positive scores $\{s^+_i\}$.

Let $opt_i$ be the optimal interleaving rank for the negative sample with the $i^{th}$ rank in the sorted list $\{s_i^*\}$ and $\mathbf{opt} = \{opt_i| j=1, \dots, \sN\}$ be the optimal interleaving rank vector. A certain subtle monotonicity property QS-suitable loss functions (see Supplementary) and the structure of the discriminant function given in (\ref{eq:jointFeature}) gives us the opportunity
to compute the interleaving rank for each negative sample independently. However, we actually need not do this computation for all the $|\calN|$ negative samples. This is because, since the interleaving rank for any negative sample can only belong to $[1,|\calP|+1]$ and $\sP << \sN$, many of the negative samples would have the same interleaving rank. This fact can be leveraged to improve the efficiency of the algorithm for finding $\mathbf{opt}$ by making use of the following observation.
\vspace{-1mm}
\begin{obs}\label{obs:greedy} If $i < j$, then $opt_i \leq opt_j$.
\end{obs}
\vspace{-1mm}
Knowing that $opt_i = opt_j$ for some $i < j$, we can conclude that $opt_i = opt_k = opt_j$ for each $i < k < j$. This provides a cheap way to compute some parts of the vector ${\bf opt}$ if an appropriate sequence is followed for computing the interleaving ranks. Even without access to the fully sorted set $\{s^*_j\}$, we can still find $s^*_j$, the $j$-highest element in $\{s^-_i\}$, for a fixed $j$, in $O( \sN)$ time. This would lead to an $O(\sP \sN)$ algorithm but we may at each step modify $\{s^-_i\}$ slowly introducing the correct order. This will make the future searches for $s^*_j$ more efficient.
 
\subsection{Divide and Conquer}
\label{subsec:divideConquer}


Algorithm \ref{alg:optranks} describes the main steps of our approach. Briefly, we 
begin by detecting $s^*_{\sN/2}$ that is the median score among the negative samples. We use this to compute $opt_{\sN/2}$. Given $opt_{\sN/2}$, we know that for all $j < \sN/2$, $opt_j \in [1,opt_{\sN/2}]$ and for all $j > \sN/2$, $opt_j \in [opt_{\sN/2}, \sP+1]$. This observation allows us to employ a divide-and-conquer recursive approach.

In more detail, we use two classical linear time array manipulating procedures {\sc Median} and {\sc Select}.
The first one outputs the index of the median element. The second one takes as its input an index of a particular element $x$. It rearranges the array such that $x$ separates higher-ranked elements from lower-ranked elements (in some total order).  For example, if array $s^-$ contains six scores $[a\;b\; 4.5 \; 6 \; 1\; c]$
then \Call{Median}{$3$, $5$} would return $3$ (the index of score $4.5$), while calling \Call{Select}{$3$, $3$, $5$} would rearrange the array to $[a\;b\; 1\; 4.5 \; 6 \; c]$ and return $4$ (the new index of $4.5$). The {\sc Select} procedure is a subroutine of the classical {\sc quicksort} algorithm. 

Using the two aforementioned procedures in conjunction with the divide-and-conquer strategy allows us to compute the entire interleaving rank vector ${\bf opt}$ and this in turn allows us to compute the semi-gradient $\nabla_{\bw} J(\bw)$, as in equation~(\ref{eq:hingeUB_grad}), efficiently.

\begin{algorithm}[t]
\caption{Recursive procedure for finding all interleaving ranks.} \label{alg:optranks}
\SetAlgoLined
\DontPrintSemicolon

\nonl 
{\bf Description:} The function finds optimal interleaving rank for all 
 $i\in[\LeftInd,\RightInd]$
given that \\
(i) array $s^-$ is partially sorted, namely
${\tt MAX}(s^-[1\dots\LeftInd-1]) \le {\tt MIN}(s^-[\LeftInd\dots\RightInd])$
and
${\tt MAX}(s^-[\LeftInd\dots\RightInd]) \le {\tt MIN}(s^-[\RightInd+1\dots\sN])$; \\
\nonl
(ii) optimal interleaving ranks for  $i\in[\LeftInd,\RightInd]$ lie in the interval $[\Left, \Right]$.
\vspace{3pt}\;

\Function{\em\Call{OptRanks}{int \LeftInd, int \RightInd, int \Left, int \Right}}{
\If{$\Left = \Right$}{set $opt_i = \Left$ for each $i\in[\LeftInd,\RightInd]$ and return}
$\Med =$ \Call{Median}{$\LeftInd$, $\RightInd$} \MyAlgComment{gives the index of the \nonl\\\hspace{15pt} median score in a subarray of $s^-$}\;
$\Med =$ \Call{Select}{$\Med$, $\LeftInd$, $\RightInd$} \MyAlgComment{splits the subarray \nonl\\\hspace{15pt} by $s=s^-[\Med]$, returns the new index of $s$}\;
Find $opt_{\Med}$ by trying all options in $[\Left, \Right]$\;
{\bf if} $\LeftInd < \Med$ {\bf then} \Call{OptRanks}{\LeftInd, \Med$-1$, \Left, $opt_{\Med}$}\;
{\bf if} $\Med < \RightInd$ {\bf then} \Call{OptRanks}{\Med$+1$, \RightInd, $opt_{\Med}$, \Right}
}
\end{algorithm}

Figure~\ref{fig:qs_flavored_algo_eg} provides an illustrative example of our divided-and-conquer strategy. Here, $|\calN|=11$ and $|\calP|=2$. We assume that the optimal interleaving rank vector $\mathbf{opt}$ is $[1,2,2,2,2,2,2,2,2,3,3]$. Let us now go through the procedure in which Algorithm~\ref{alg:optranks} computes this optimal interleaving rank vector. Before starting the recursive procedure, we only sort the positive samples according to their scores and do not sort the negative samples. To start with, we call \Call{OptRanks}{$1,11,1,3$}. We find the negative sample with the median score ($6^{th}$ highest in this case) and compute its optimal interleaving rank $opt_6$ to be 2. In the next step of the recursion, we make the following calls: \Call{OptRanks}{$1,5,1,2$} and \Call{OptRanks}{$7,11,2,3$}. These calls compute $opt_3$ and $opt_9$ to be 2. In the next set of recursion calls however, the calls \Call{OptRanks}{$4,5,2,2$} and \Call{OptRanks}{$7,8,2,2$}, get terminated in step 4 of Algorithm~\ref{alg:optranks} and $opt_j$ for $j={4,5,7,8}$ are assigned without any additional computation. We then continue this procedure recursively for progressively smaller intervals as described in Algorithm~\ref{alg:optranks}. Leveraging the fact stated in observation~\ref{obs:greedy}, our algorithm has to explicitly compute the interleaving rank for only 6 (shown in square brackets) out of the 11 negative samples. In a typical real data set, which is skewed more in favor of the negative samples, the expected number of negative samples for which is the interleaving rank has to be explicitly computed is far less than $|\calN|$. In contrast, the algorithm proposed by Yue {\em et al.\ }in \cite{yuesigir07} first sorts the entire negative set in descending order
of their scores and explicitly computes the interleaving rank for each of the $|\calN|$ negative samples.

\begin{figure}[t]
\begin{center}
\begin{tabular}{ c c c c c c c c c c c }
  \_ & \_ & \_ & \_ & \_ & \_ & \_ & \_ & \_ & \_ & \_ \\
  \_ & \_ & \_ & \_ & \_ & \underline{[2]} & \_ & \_ & \_ & \_ & \_ \\
  \_ & \_ & \underline{[2]} & \_ & \_ & \underline{[2]} & \_ & \_ & \underline{[2]} & \_ & \_ \\
  \_ & \underline{[2]} & \underline{[2]} & \underline{2} & \underline{2} & \underline{[2]} & \underline{2} & \underline{2} & \underline{[2]} & \underline{[3]} & \_ \\
  \underline{[1]} & \underline{[2]} & \underline{[2]} & \underline{2} & \underline{2} & \underline{[2]} & \underline{2} & \underline{2} & \underline{[2]} & \underline{[3]} & \underline{3} \\
  &  &  &  &  &  &  &  &  &  &  \\
\end{tabular}
\caption{\em Example illustrating the path followed by the quick sort flavored recursive algorithm while computing the interleaving rank vector $\mathbf{opt}$. Row correspond to the status of $\mathbf{opt}$ at selected time steps.\vspace*{-7mm}}
\label{fig:qs_flavored_algo_eg}
\end{center}
\end{figure}
\vspace*{-0.5mm}
\subsection{Computational Complexity}
\label{subsec:complexity}
\vspace*{-0.5mm}
The computational complexity of the divide-and-conquer strategy to estimate the output of problem~(\ref{eq:LAI}), is given by the following theorem.
\vspace*{-1.5mm}
\begin{theorem}\label{alg} If $\Delta$ is QS-suitable, then the task \eqref{eq:LAI} can be solved in time $O(\sN \log\sP + \sP \log\sP + \sP \log\sN)$, which in the most common case $\sN > \sP$ reduces to $O(\sN \log\sP)$ and any comparison-based algorithm would require $\Omega(\left|\calN\right| \log \left|\calP\right|)$ operations.
\end{theorem}
\vspace*{-1mm}
\begin{proof}
Please refer to Appendix (Supplementary).
\end{proof}

Note that the above theorem not only establishes the superior runtime of our approach ($O(\sN \log\sP)$ compared to $O(\sN \log\sN)$ of \cite{yuesigir07} and \cite{mohapatranips14}), it also provides an asymptotic lower bound for comparison based algorithms. However, it does not rule out the possibility of improving the constants hidden within the asymptotic notation for a given loss function. For example, as mentioned earlier, one can exploit the additional structure of the {\sc ap} loss, as presented in \cite{mohapatranips14}, to further speed-up our algorithm.
\vspace*{-1mm}

\section{Experiments}
\label{sec:expt}

We demonstrate the efficacy of our approach on three vision tasks with increasing level of complexity. First, we use the simple experimental setup of doing action classification on the {\sc pascal voc} 2011 data set using a shallow model. This experimental set up allow us to thoroughly analyze the performance of our method as well as the baselines by varying the sample set sizes. Second, we apply our method to a large scale experiment of doing object detection on the {\sc pascal voc} 2007 data set using a shallow model. This demonstrates that our approach can be used in conjunction with a large data set consisting of millions of samples. Finally, we demonstrate the effectiveness of our method for layer wise training of a deep network on the task of image classification using the {\sc cifar}-10 data set.

\subsection{Action Classification}
\vspace*{-0.5mm}

\myparagraph{Data set.}
We use the {\sc pascal voc} 2011 \cite{everinghamijcv10} action classification data set for our experiments. This data set consists of 4846 images, which include
10 different action classes. The data set is divided into two parts: $3347$ `trainval' person bounding boxes
and $3363$ `test' person bounding boxes. We use the `trainval' bounding boxes for training since
their ground-truth action classes are known. We evaluate the accuracy of the different models on the `test' bounding boxes using the {\sc pascal} evaluation server.

\myparagraph{Model.}
We use structured {\sc svm} models as discriminant functions and use the standard poselet \cite{majicvpr11} activation features to define the sample feature for each person bounding box. The feature vector consists of 2400 action
poselet activations and 4 object detection scores. We refer the reader to~\cite{majicvpr11} for details regarding the feature vector.

\myparagraph{Methods.} We show the effectiveness of our method in optimizing both {\sc ap} loss and {\sc ndcg} loss to learn the model parameters. Specifically, we report the computational time for the loss-augmented inference evaluations. For {\sc ap} loss, we compare our method (referred to as {\sc ap\_qs}) with the loss-augmented inference procedure described in \cite{yuesigir07} (referred to as {\sc ap}). For {\sc ndcg} loss, we compare our method (referred to as {\sc ndcg\_qs}) with the loss-augmented inference procedure described in \cite{chakrabartikdd08} (referred to as {\sc ndcg}). We also report results for loss-augmented inference evaluations when using the simple decomposable 0-1 loss function (referred to as {\sc 0-1}). The hyperparameters involved are fixed using 5-fold cross-validation on the `trainval' set. 

\begin{table}[b]
\center{
\scriptsize
\begin{tabular}{|c|c|c|c|c|}
\hline
0-1 &  {\sc ap} & {\sc ap\_qs} &  {\sc ndcg} &  {\sc ndcg\_qs}\\
\hline
0.0694 & 0.7154 & 0.0625 & 6.8019 & 0.0473\\ 
\hline
\end{tabular}
}

\caption{\emph{Total computation time (in seconds) when using the different methods, for multiple calls to loss augmented inference during model training. The reported time is averaged over the training for all the action classes.}}
\label{table:ap_timing}
\end{table}

\begin{table}[b]
\center{
\scriptsize
\begin{tabular}{|c|c|c|c|c|}
\hline
0-1 &  {\sc ap} & {\sc ap\_qs} &  {\sc ndcg} &  {\sc ndcg\_qs}\\
\hline
0.48$\pm$0.03 & 16.29$\pm$0.18 & 1.48$\pm$0.39 & 71.07$\pm$1.57 & 0.55$\pm$0.11\\ 
\hline
\end{tabular}
}

\caption{\emph{Mean computation time (in milli-seconds) when using the different methods, for single call to loss augmented inference. The reported time is averaged over all training iterations and over all the action classes.}}
\label{table:ap_timing_per_iter}
\end{table}

\begin{figure*}
\begin{center}
\begin{tabular}{@{\hspace{0pt}}c@{\hspace{0pt}}c@{\hspace{0pt}}c}
\includegraphics[trim = 0mm 0mm 0mm 0mm, clip, width=0.33\linewidth]{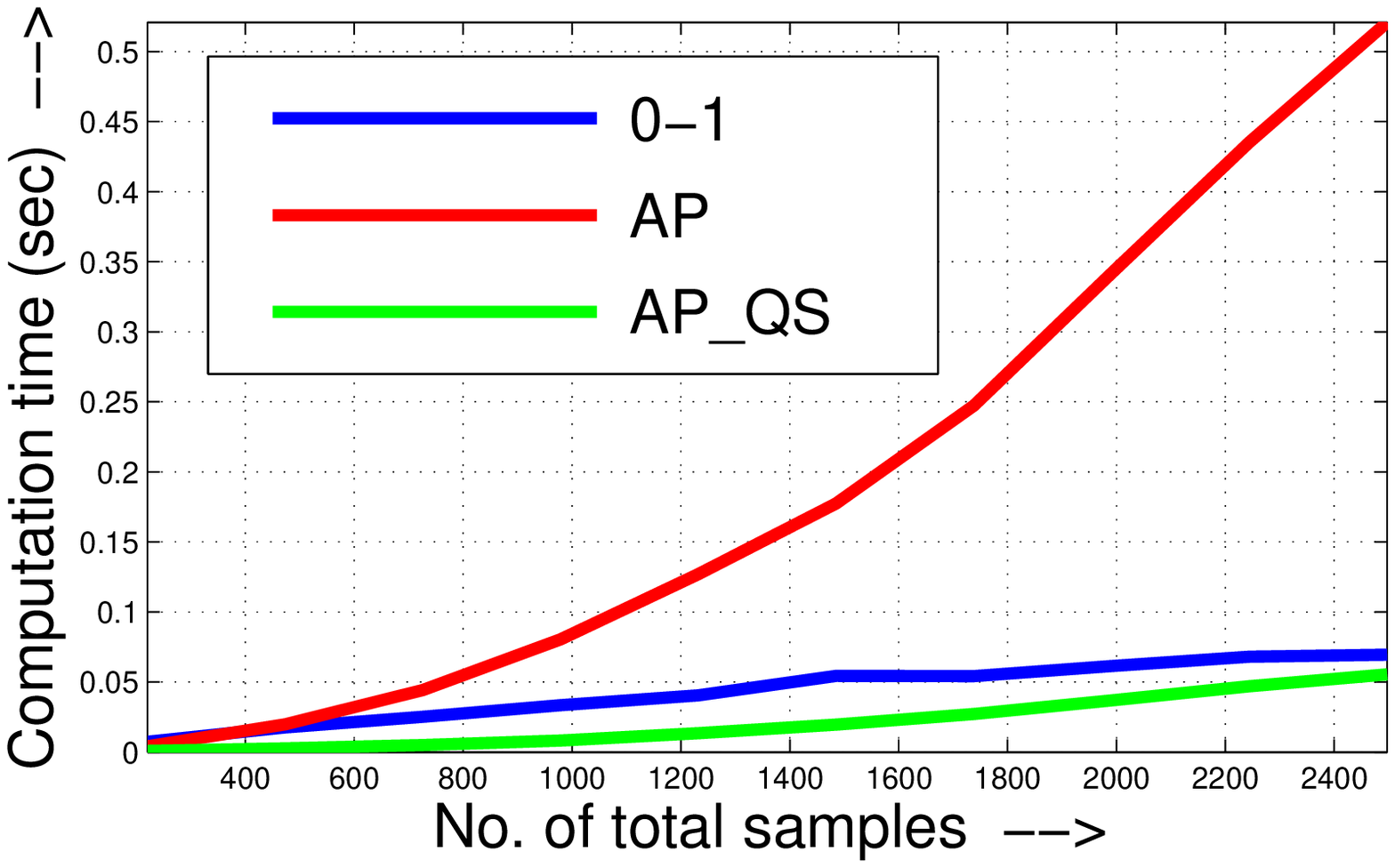}
\includegraphics[trim = 0mm 0mm 0mm 0mm, clip, width=0.33\linewidth]{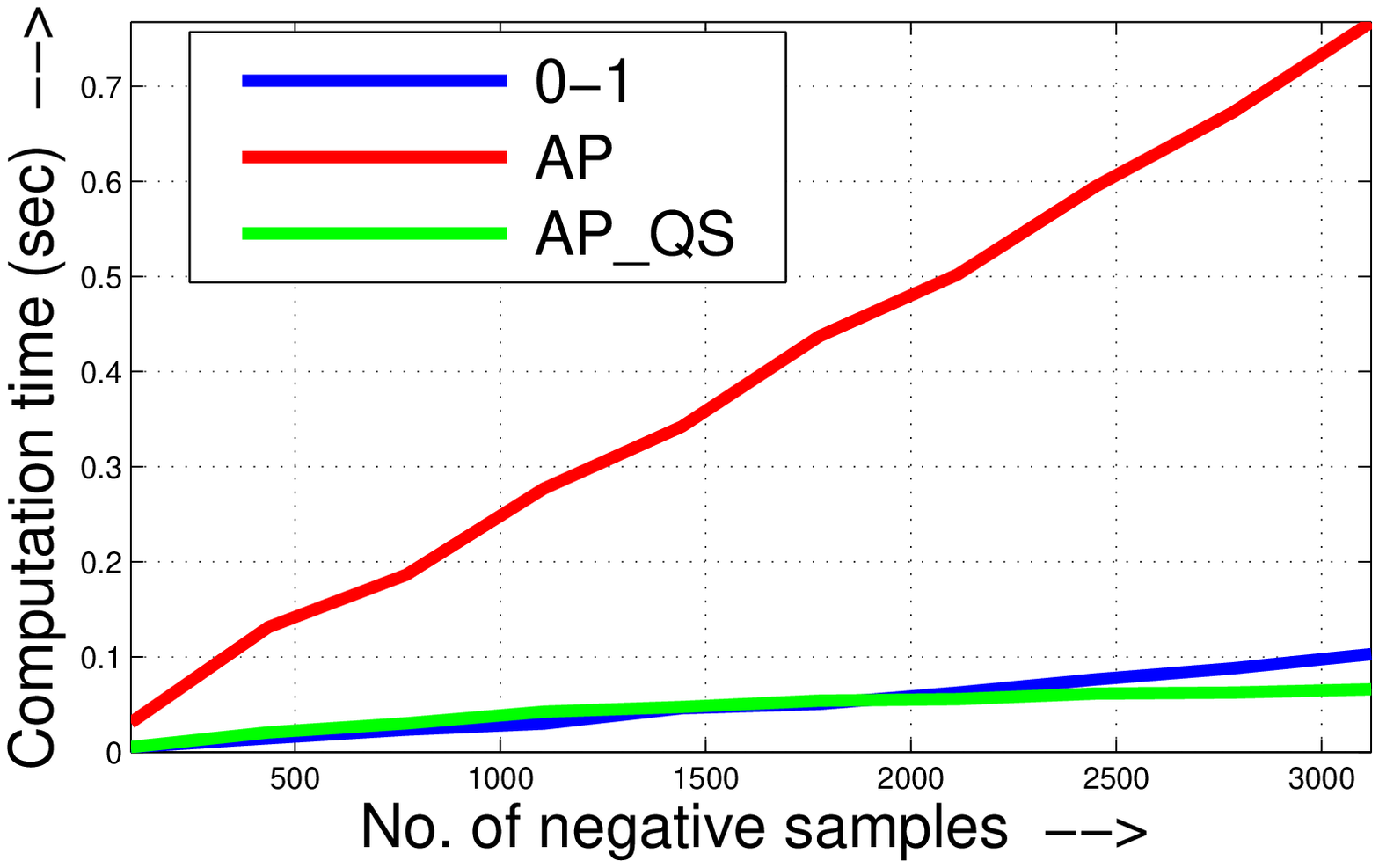} 
\includegraphics[trim = 0mm 0mm 0mm 0mm, clip, width=0.33\linewidth]{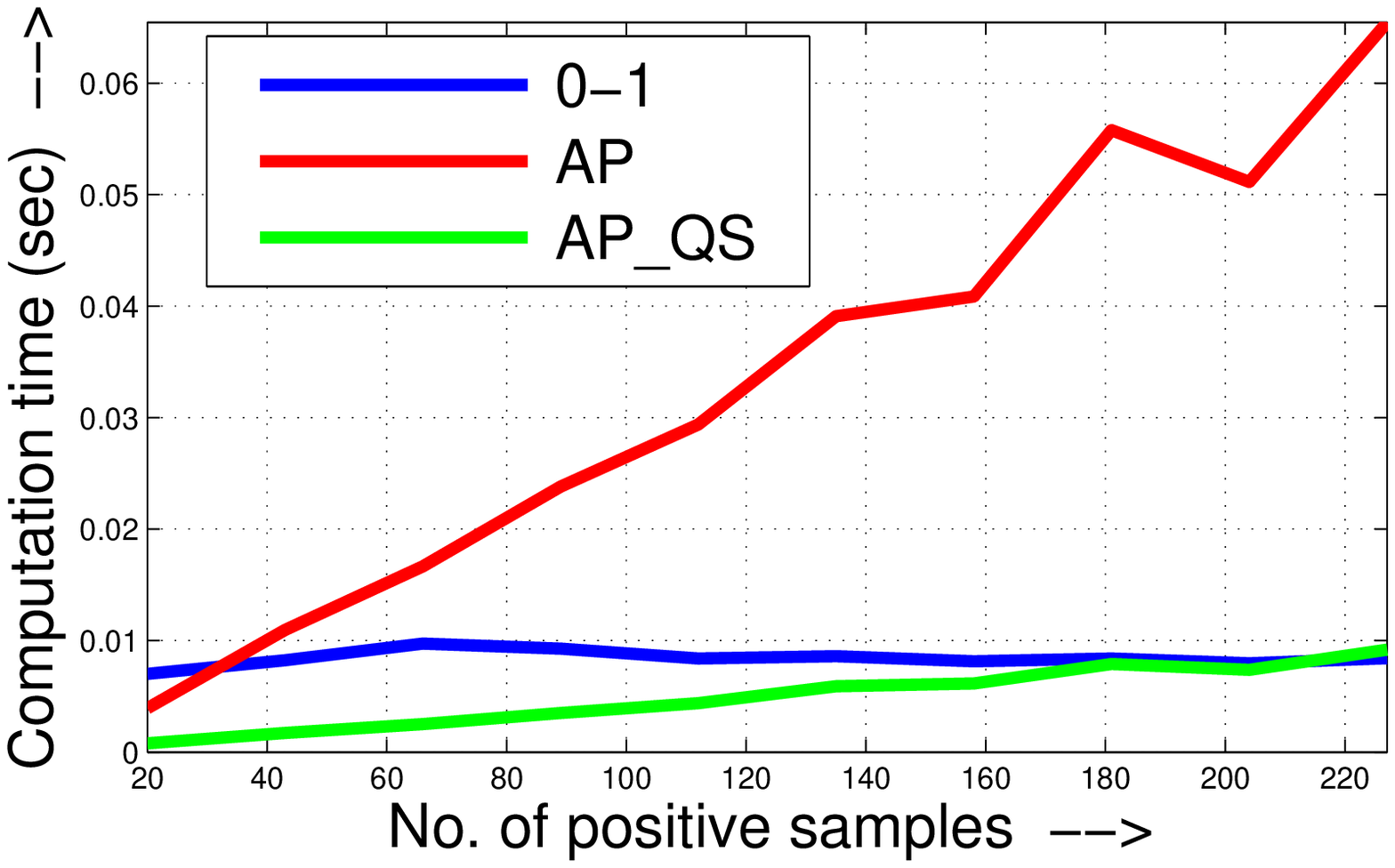}
\end{tabular}
\end{center}
\caption{\emph{Total computation time for multiple calls to loss augmented inference during model training, when the number of total, negative and positive samples are varied.  Here, 0-1, {\sc ap} and {\sc ap\_qs} correspond to loss augmented inference procedures for 0-1 loss, for {\sc ap} loss using \cite{yuesigir07} and for {\sc ap} loss using our method respectively. It can be seen that our method scales really well with respect to sample set sizes and takes computational time that is comparable to what is required for simpler 0-1 decomposable loss.}\vspace*{-3mm}}
\label{fig:apsvm_comptime_inference_all}
\end{figure*}

\begin{figure*}
\begin{center}
\begin{tabular}{@{\hspace{0pt}}c@{\hspace{0pt}}c@{\hspace{0pt}}c}
\includegraphics[trim = 0mm 0mm 0mm 0mm, clip, width=0.33\linewidth]{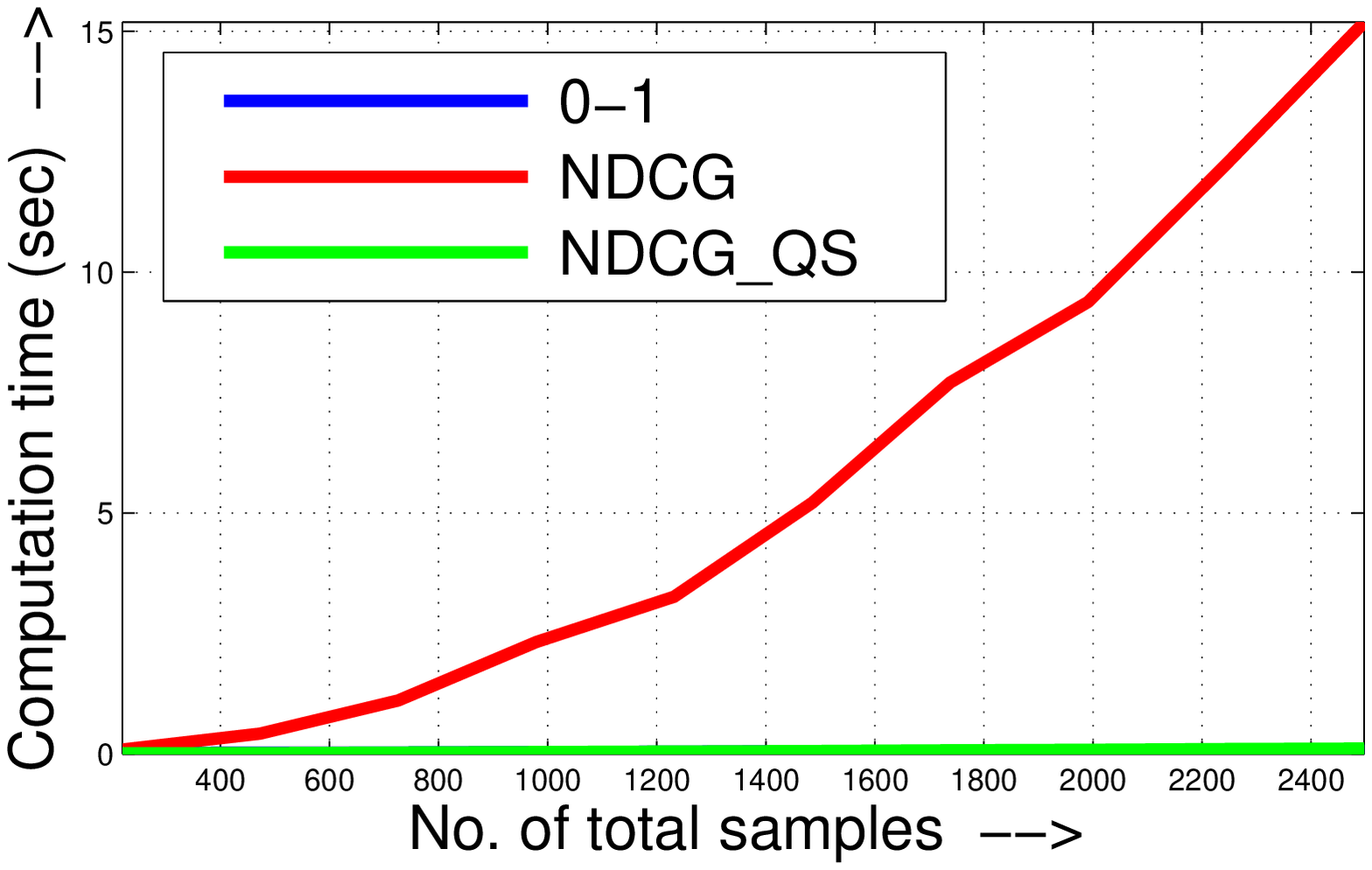}
\includegraphics[trim = 0mm 0mm 0mm 0mm, clip, width=0.33\linewidth]{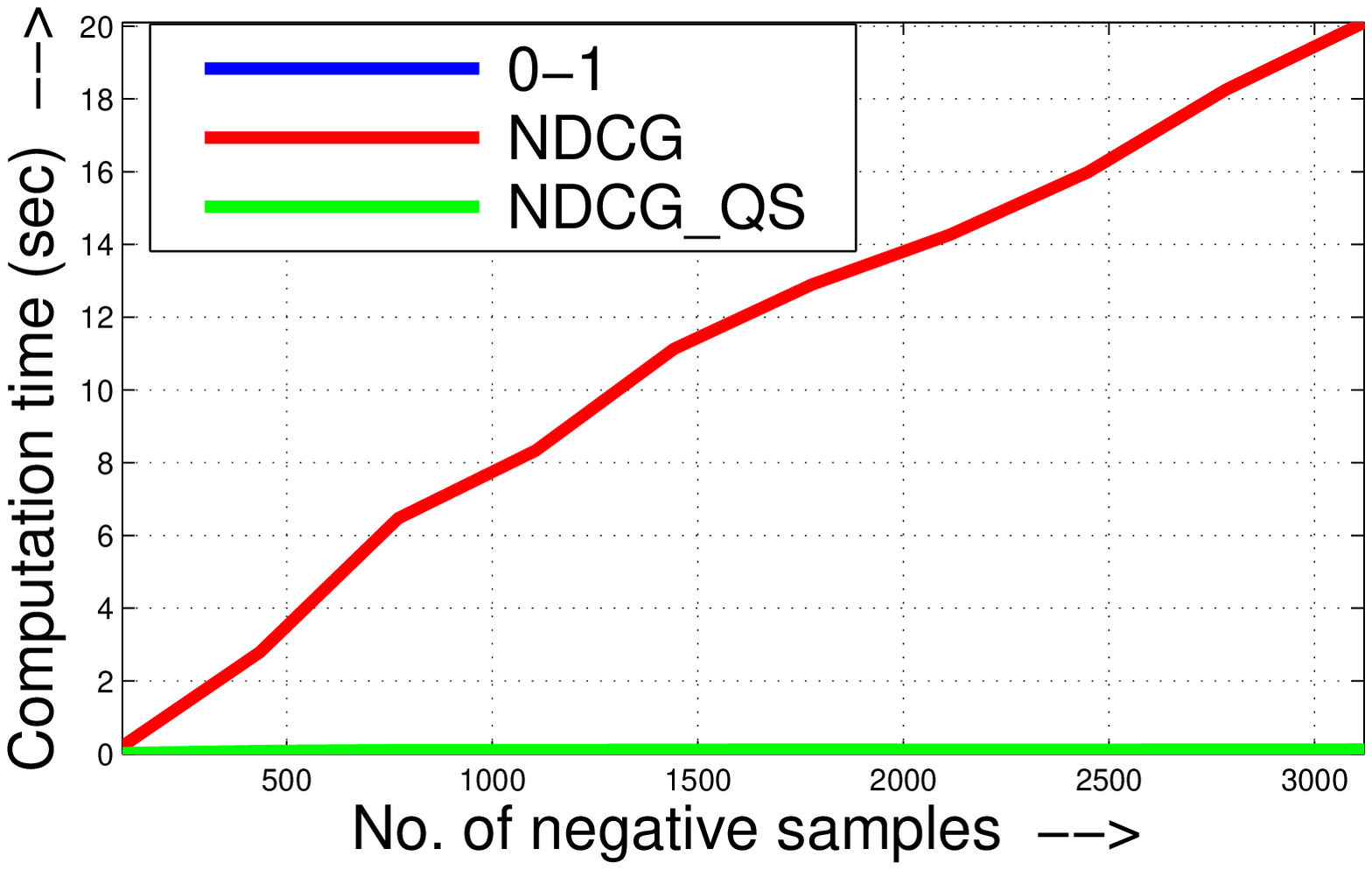} 
\includegraphics[trim = 0mm 0mm 0mm 0mm, clip, width=0.33\linewidth]{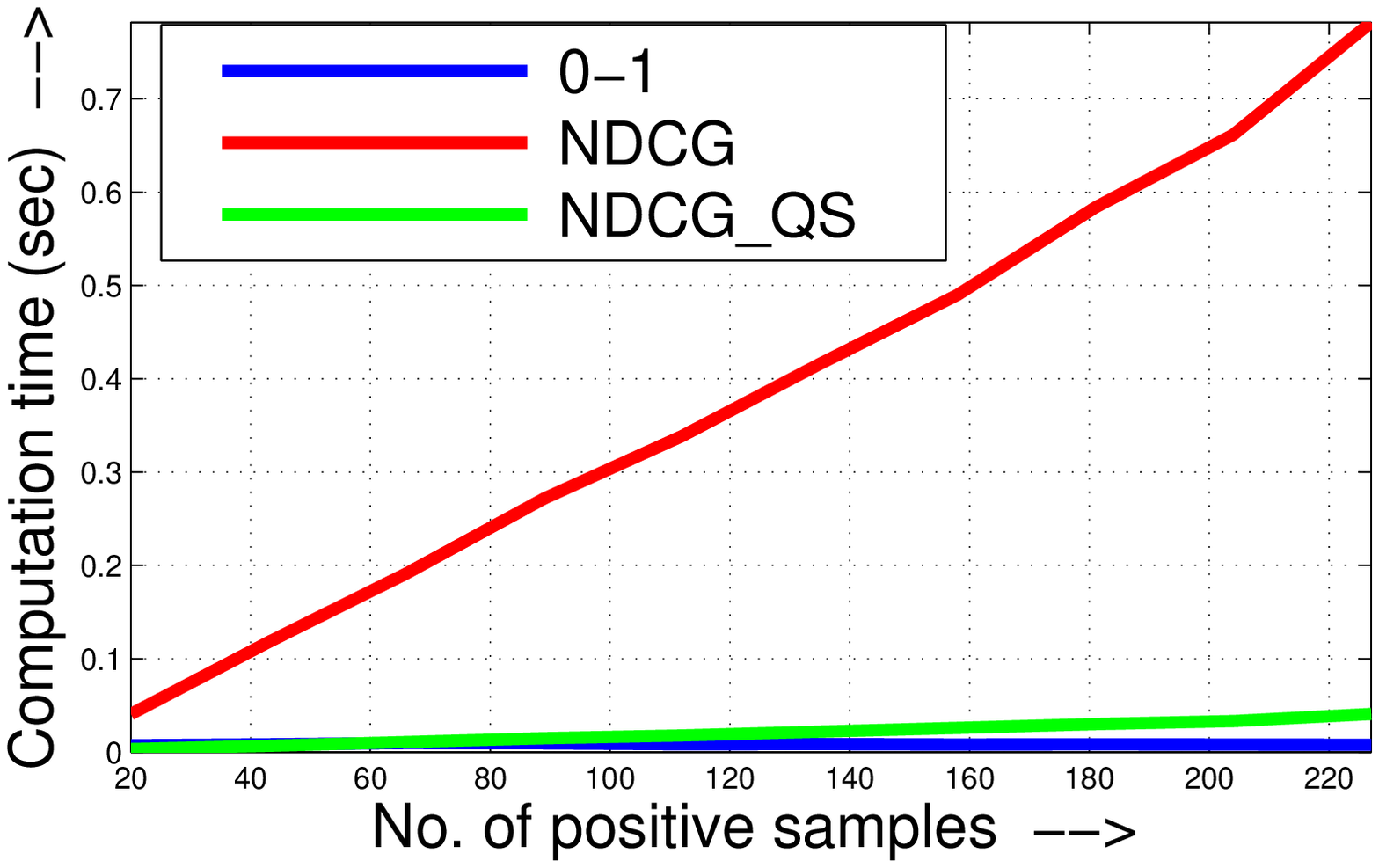}
\end{tabular}
\end{center}
\caption{\emph{Total computation time for multiple calls to loss augmented inference during model training, when the number of total, negative and positive samples are varied. Here, 0-1, {\sc ndcg} and {\sc ndcg\_qs} correspond to loss augmented inference procedures for 0-1 loss, for {\sc ndcg} loss using \cite{chakrabartikdd08} and for {\sc ndcg} loss using our method respectively. As can be seen, our approach scales elegantly with respect to sample set sizes and is comparable to the simpler 0-1 decomposable loss in terms of computation time.}\vspace*{-3mm}}
\label{fig:ndcgsvm_comptime_inference_all}
\end{figure*} 

\myparagraph{Results.}
When we minimize {\sc ap} loss on the training set to learn the model parameters, we get a mean {\sc ap} of 51.196 on the test set. In comparison, minimizing 0-1 loss to learn model parameters leads to a mean {\sc ap} value of 47.934 on the test set. Similarly, minimizing {\sc ndcg} loss for parameter learning gives a superior mean {\sc ndcg} value of 85.521 on the test set, compared to that of 84.3823 when using 0-1 loss. The {\sc ap} and {\sc ndcg} values obtained on the test set for individual action classes can be found in the supplementary material. This clearly demonstrates the usefulness of directly using rank based loss functions like {\sc ap} loss and {\sc ndcg} loss for learning model parameters, instead of using simple decomposable loss functions like 0-1 loss as surrogates.

The time required for the loss augmented inference evaluations, while optimizing the different loss functions for learning model parameters, are shown in Table~\ref{table:ap_timing}. It can be seen that using our method ({\sc ap\_qs}, {\sc ndcg\_qs}) leads to reduction in computational time by a factor of more than $10$, when compared to the methods proposed in \cite{yuesigir07}  and \cite{chakrabartikdd08} for {\sc ap} loss and {\sc ndcg} loss respectively. For {\sc ap} loss, the method proposed in \cite{mohapatranips14} takes computational time of 0.0985 sec for loss augmented inference. Note that, this method is specific to {\sc ap} loss, but our more general method is still around 3 times faster. It can also be observed that although the computational time for each call to loss-augmented inference for 0-1 loss is slightly less than that for {\sc ap} loss and {\sc ndcg} loss (Table ~\ref{table:ap_timing_per_iter}), in some cases we observe that we required more calls to optimize the 0-1 loss. As a result, in those cases training using 0-1 loss is slower than training using {\sc ap} or {\sc ndcg} loss with our proposed method. 

In order to understand the effect of the size and composition of the data set on our approaches, we perform 3 experiments with variable number of samples for the action class ’phoning’. First, we vary the total number of samples while fixing the positive to negative ratio to $1:10$. Second, we vary the number of negative samples while fixing the number of positive samples to 227. Third, we vary the number of positive samples while fixing the number of negative samples to 200. As can be seen in Fig. \ref{fig:apsvm_comptime_inference_all} and Fig. \ref{fig:ndcgsvm_comptime_inference_all}, the time required for loss-augmented inference is significantly lower using our approach for both {\sc ap} and {\sc ndcg} loss.

\subsection{Object Detection}

\myparagraph{Data set.}
We use the {\sc pascal} {\sc voc} 2007 \cite{everinghamijcv10} object detection data set, which consists of 
a total of 9963 images. The data set is divided into a `trainval' set of 5011 images and a `test' set of 4952 images.
All the images are labeled to indicate the presence or absence of the instances of 20 different object categories. In addition, we are
also provided with tight bounding boxes around the object instances, which we ignore during training and testing. Instead, we treat
the location of the objects as a latent variable. In order to reduce the latent variable space, we use the selective-search
algorithm~\cite{uijlingsijcv13} in its fast mode, which generates an average of 2000 candidate windows per image. This results in a training set size of approximately 10 million bounding boxes.

\myparagraph{Model.}
For each candidate window, we use a feature representation that is extracted from a trained Convolutional Neural Network ({\sc cnn}). 
Specifically, we pass the image as input to the {\sc cnn} and use the activation vector of the penultimate layer of the {\sc cnn} as the feature vector.
Inspired by the {\sc r-cnn} pipeline of Girshick {\em et al.\ }~\cite{girshickcvpr14}, we use the {\sc cnn} that is trained on the ImageNet data set~\cite{dengcvpr09},
by rescaling each candidate window to a fixed size of $224\times224$. The length of the resulting feature vector is $4096$. However, in contrast to \cite{girshickcvpr14}, we do not assume ground-truth bounding boxes to be available for training images. We instead optimize {\sc ap} loss in a weakly supervised framework to learn the parameters of the {\sc svm} based object detectors for the 20 object categories.

\myparagraph{Methods.}
We use our approach to learn the parameters of latent {\sc ap-svm}s \cite{behlcvpr14} for each object category. In our experiments, we fix the hyperparameters using 5-fold cross-validation. During testing, we evaluate each candidate window generated by selective search and use non-maxima suppression to prune highly overlapping detections.

\myparagraph{Results.}
For this task of weakly supervised object detection, using {\sc ap} loss for learning model parameters leads to a mean test {\sc ap} of 36.616 which is significantly better than the 29.4995 obtained using 0-1 loss. The {\sc ap} values obtained on the test set by the detectors for each object class can be found in the supplementary material. These results establish the usefulness of optimizing {\sc ap} loss for learning the object detectors. On the other hand, optimizing {\sc ap} loss for this task places high computational demands due to the size of the data set (5011 `trainval' images) as well as the latent space (2000 candidate windows per image) amounting to around 10 million bounding boxes. We show that using our method for loss-augmented inference ({\sc lai}) leads to significant saving in computational time. During training, the total time taken for {\sc lai}, averaged over all the 20 classes, was 0.5214 sec for our method which is an order of magnitude better than the 7.623 sec taken by the algorithm proposed in \cite{yuesigir07}. Thus, using our efficient quicksort flavored algorithm can be critical when optimizing non-decomposable loss functions like {\sc ap} loss for large scale data sets.

\subsection{Image Classification}

\myparagraph{Data set.}
We use the {\sc cifar}-10 data set \cite{krizhevsky2009learning}, which consists of 
a total of 60,000 images of size $32\times32$ pixels. Each image belongs to one of 10 specified classes. The data set is divided into a `trainval' set of 50,000 images and a `test' set of 10,000 images. From the 50,000 `trainval' images, we use 45,000 for training and 5,000 for validation. For our experiments, all the images are centered and normalized.

\myparagraph{Model.}
We use a deep neural network as our classification model. Specifically, we use a piecewise linear convolutional neural network ({\sc pl-cnn}) as proposed in \cite{berrada2016trusting}. We follow the same framework as \cite{berrada2016trusting} for experiments on the {\sc cifar}-10 data set and use a {\sc pl-cnn} architecture comprising 6 convolutional layers and an {\sc svm} last layer. For all our experiments, we use a network that is pre-trained using softmax and cross-entropy loss.

\myparagraph{Methods.}
We learn the weights of the {\sc pl-cnn} by optimizing {\sc ap} loss and {\sc ndcg} loss for the training data set. For comparison, we also report results for parameter learning using the simple decomposable 0-1 loss. We use the layerwise optimization algorithm called {\sc lw-svm}, proposed in \cite{berrada2016trusting}, for optimizing the different loss functions with respect to the network weights. Following the training regime used in \cite{berrada2016trusting}, we warm start the optimization with a few epochs of Adadelta \cite{zeilerAdadelta} before running the layer wise optimization. The {\sc lw-svm} algorithm involves solving a structured {\sc svm} problem for one layer at a time. This requires tens of thousands of calls to loss augmented inference and having a efficient procedure is therefore critical for scalability. We compare our method for loss-augmented inference with the methods described in \cite{yuesigir07} and \cite{chakrabartikdd08}, for {\sc ap} loss and {\sc ndcg} loss respectively. 

\myparagraph{Results.}
We get a better mean {\sc ap} of $85.28$ on the test set when we directly optimize {\sc ap} loss for learning network weights compared to that of $84.22$ for 0-1 loss. Similarly, directly optimizing {\sc ndcg } loss leads to a better mean {\sc ndcg} of $96.14$ on the test set compared to $95.31$ for 0-1 loss. This establishes the usefulness of optimizing non-decomposable loss functions like the {\sc ap} loss and {\sc ndcg} loss. The {\sc lw-svm} algorithm involves very high number of calls to the loss augmented inference procedure. In light of this, the efficient method for loss augmented inference proposed in this paper leads to significant reduction in total training time. When optimizing the {\sc ap} loss, using our method leads to a total training time of $1.589$ hrs compared to that of $1.974$ hrs for the algorithm proposed in \cite{yuesigir07}. Similarly, when optimizing {\sc ndcg} loss, our method leads to a total training time of $1.632$ hrs, which is significantly better than the $2.217$ hrs taken for training when using the method proposed in \cite{chakrabartikdd08}. This indicates that using our method helps the layerwise training procedure scale much better.

\section{Discussion}

We provided a characterization of ranking based loss functions that are amenable to a quicksort based optimization algorithm for the loss augmented inference problem. We proved that the our algorithm provides a better computational complexity than the state of the art methods for
{\sc ap} and {\sc ndcg} loss functions and also established that the complexity of our algorithm cannot be improved upon asymptotically by any comparison based
method. We empirically demonstrated the efficacy of our approach on challenging real world vision problems. In future, we would like to explore extending our approach to other ranking based non-decomposable loss functions like those based on the F-measure or the mean reciprocal rank.

\section*{Acknowledgement}
\vspace*{-2mm}
This work is partially funded by the EPSRC grants EP/P020658/1 and TU/B/000048 and a CEFIPRA grant. MR and VK were supported by the European Research Council under the European Unions Seventh Framework Programme (FP7/2007-2013)/ERC grant agreement no 616160. 

\setcounter{section}{0}
\def\LeftInd{\ensuremath{\ell^-}}
\def\RightInd{\ensuremath{r^-}}
\def\Left{{\ensuremath{\ell^+}}}
\def\Right{{\ensuremath{r^+}}}

\def\Med{{\ensuremath{m}}}

\section*{Appendix}

The supplementary material is organized as follows. In Section \ref{sec:qs-suitable_complete}, we give a full definition of {\sc qs}-suitable loss functions and in Section \ref{sec:algo_detail} we justify the correctness of Algorithm $1$. Section \ref{sec:lossprops} contains proofs of {\sc qs}-suitability of {\sc ap} and {\sc ndcg} losses and Section \ref{sec:complexity} establishes worst-case complexity of Algorithm$~1$ as well as a matching lower-bound on the complexity of solving Problem ($6$). Several remaining proofs are delegated to Section \ref{sec:proof} and we use Section \ref{sec:NDCG} for certain clarifications regarding previous use of {\sc ndcg} in the literature. Finally, we report some additional experimental results in Section \ref{sec:addexpt}. 

\section{Complete characterization of QS-Suitable Loss Functions}
\label{sec:qs-suitable_complete}

A proper loss function $\Delta = \Delta(\bR^*, \bR)$ is called \emph{{\sc qs}-suitable} if it meets the following three conditions.
\begin{enumerate}
\item[(C1)] {\bf Negative decomposability with interleaving dependence.} There are functions $\delta_j \colon \{1, \dots, \sP+1\} \to \mathbb{R}$ for $j = 1, \dots, \sN$ such that for a proper ranking $\bR$ one can write 
$$\Delta(\bR^*, \bR) = \sum_{\bx \in \calN} \delta_{ind^-(\bx)}(rank(\bx)).$$
\item[(C2)] {\bf $j$-monotonicity of discrete derivative.} For every $1 \leq j < \sN$ and $1 \leq i \leq \sP$ we have
$$\delta_{j+1}(i+1) - \delta_{j+1}(i) \geq \delta_{j}(i+1)- \delta_j(i).$$
\item[(C3)] {\bf Fast evaluation of discrete derivative.} For any $j \in \{1, \dots, \sN\}$ and $i \in \{1, \dots, \sP \}$, can the value
$\delta_j(i+1) - \delta_j(i)$
be computed in constant time. 
\end{enumerate}

From (C1), we can see that the loss function depends only on the interleaving ranks of the negative samples. More accurately, it depends on the vector $\mathbf{r} = (r_1, \dots, r_{|\calN|})$ where $r_i$ is the interleaving rank of the $i$-th most relevant negative sample (i.e. with the $i$-th highest score).

Another way to interpret this type of dependence is by looking at the $\pm$-pattern of a ranking which can be obtained as follows.
Given a proper ranking $\bR$ (in the form of a permutation of samples), it is the pattern obtained by replacing each positive sample with a ``$+$'' symbol and each negative sample with a ``$-$'' symbol. It is easy to see that the $\pm$-pattern uniquely determines the vector $\mathbf{r}$ and vice versa and thus (C1) also implies dependence on the $\pm$-pattern.

\section{Justification of Algorithm $1$} \label{sec:algo_detail}

First key point is that the entire objective function ($6$) inherits properties (C1) and (C2).

\begin{obs}\label{obs:obj_nice} The following holds:
\begin{enumerate} 
\item[(a)]
\label{obs:fj}
There are functions $f_j \colon \{1, \dots, \sP+1\} \to \mathbb{R}$ for $j = 1, \dots, \sN$ such that the objective function in ($6$) can be written as
$$\sum_{j=1}^{\sN} f_{j}(r_j),$$ 
where $r_j$ is the interleaving rank of the negative sample $x$ with $ind^-(x) = j$.
\item[(b)] The functions $f_j$ inherit property (C2). More precisely, for every $1 \leq j < \sN$ and $1 \leq i \leq \sP$ we have
$$f_{j+1}(i+1) - f_{j+1}(i) \geq f_{j}(i+1)- f_j(i).$$
\item[(c)] We can compute $\argmax_{l \leq i \leq r} f_j(i)$ in $O(r-l)$ time if we are provided access to the sorted array $\{s_i^+\}$ and to the score of the negative sample $x$ with $ind^-(x) = j$.
\end{enumerate}
\end{obs}

As a result, solving Problem ($6$) reduces to computing the optimal interleaving ranks (or the optimal vector $\mathbf{r}$ from the remark above) \footnote{Note that the value of the objective can be computed efficiently given a vector $\mathbf{r}$ -- for example by constructing any ranking $\bR$ which respects $\mathbf{r}$.}. 

The next vital point is that these interleaving ranks can be optimized independently. This is however not obvious. One certainly can maximize each $f_j$  but the resulting vector $\mathbf{r}$ may not induce any ranking -- its entries may not be monotone.

But as a matter of fact, this does not happen and Observation $2$ from the main text gives the precise guarantee. This ``correctness of greedy maximization'' hinges upon condition (C2) as will also be demonstrated with a  counterexample given later in Section \ref{sec:NDCG}.

All in all, it suffices to compute the vector $\mathbf{opt}$ in which $opt_j = \max \argmax f_j$ (the maximum ensures that ties are broken consistently) as is done in the main text of the paper.

\section{Properties of $\Delta_{AP}$ and $\Delta_{NDCG}$} \label{sec:lossprops}

In this place, let us prove the aforementioned properties of $\Delta_{AP}$ and $\Delta_{NDCG}$.

\begin{proposition} $\Delta_{NDCG}$ is QS-suitable.
\end{proposition}
\begin{proof} As for (C1), let us first verify that the functions $\delta_j$ can be set as
$$\delta_j(i) = \frac{1}{C} \left(D(i+j-1) - D(\sP + j)\right),$$
where $C = \sum_{i=1}^{\sP} D(i)$.
Indeed, one can check that
\begin{align*}
&\Delta(\bR^*, \bR) \\ &= 1-\frac{\sum_{\bx \in {\cal P}} D(ind(\bx))}{\sum_{i = 1}^{\sP} D(i)} \\
 &= \frac{1}{C} \sum_{i = 1}^{\sP} D(i) - \sum_{\bx \in {\cal P}} D(ind^+(\bx) + rank(\bx)
-1 )  \\
&= \frac{1}{C\!} \sum_{\bx \in {\cal N}}\!\! D(ind^-\!(\bx)\! + \!rank(\bx)\! -\! 1 )\! - \! D(\sP \!+\! ind^-\!(\bx)) \\
 &= \sum_{\bx \in \calN} \delta_{ind^-(\bx)}(rank(\bx))
\end{align*}
as desired.
As for (C2) and (C3), let us realize that
$$\delta_j(i+1) - \delta_j(i) = \frac{1}{C} \left( D(i+j) - D(i+j-1) \right).$$
Then (C3) becomes trivial and checking (C2) reduces to
$$D(i+j+1) + D(i+j-1) \geq 2 D(i+j)$$
which follows from convexity of the function $D$.
\end{proof}

\begin{proposition} $\Delta_{AP}$ is QS-suitable.
\end{proposition}
\begin{proof} Regarding (C1), the functions $\delta_j$ were already identified in \cite{yuesigir07} as
$$\delta_j(i) = \frac{1}{\sP} \sum_{k=i}^{\sP}\left( \frac{j}{j+k} - \frac{j-1}{j+k-1} \right)$$
so after writing
$$\delta_j(i+1) - \delta_j(i) = \frac{j-1}{j+i-1} - \frac{j}{j+i}$$
we again have (C3) for free and (C2) reduces to
$$2g_i(j) \geq g_i(j-1) + g_i(j+1),$$
where $g_i(x) = \frac{x}{x+i}$, and the conclusion follows from concavity of $g_i(x)$ for $x > 0$.
\end{proof}

\section{Computational Complexity}
\label{sec:complexity}

Now is the time to establish the computational complexity of Algorithm $1$ as well as the afore-mentioned matching lower bound.
efficiency.

\begin{theorem}\label{alg} If $\Delta$ is QS-suitable, then the Problem ($6$) can be solved in time $O(\sN \log\sP + \sP \log\sP + \sP \log\sN)$, which in the most common case $\sN > \sP$ reduces to $O(\sN \log\sP)$.
\end{theorem}
Outside running Algorithm $1$, the entire computation also consists of preprocessing (sorting positive samples by their scores) and post processing (computing the output from vector $\mathbf{opt}$). These subroutines have only one non-linear complexity term -- $O(\sP \log\sP)$ coming from the sorting.
Therefore, it remains to establish the complexity of Algorithm $1$ as $O(\sN \log \sP + \sP \log \sN)$. 

To this end, let us denote $n = \RightInd-\LeftInd+1$ and $p = \Right-\Left+1$, and set $T_{neg}(n,p)$, $T_{pos}(n,p)$ as the total time spent traversing the arrays of negative and positive sample scores, respectively, including recursive calls. The negative score array is traversed in the {\sc Median} and {\sc Select} procedures and the positive scores are traversed when searching for $opt_{\Med}$. The latter has by complexity $O(p)$, due to Observation \ref{obs:obj_nice}(c), whose assumption are always satisfied during the run of the algorithm.

\begin{proposition} The runtimes $T_{neg}(n,p)$ and $T_{pos}(n,p)$ satisfy the following recursive inequalities
\begin{align*}
T_{neg}(n,p) &\leq Cn + T_{neg}(n/2, p_1) + T_{neg}(n/2, p_2) \\
             &                              \hspace{80pt} \text{for some} \quad p_1+p_2=p+1, \\
T_{pos}(n,p) &\leq Cp + T_{pos}(n/2, p_1) + T_{pos}(n/2, p_2) \\
             &                              \hspace{80pt} \text{for some} \quad p_1+p_2=p+1, \\
\noalign{\centering $T_{neg}(n,1) \leq Cn, \qquad T_{neg}(1,p) = 0,$ }
\noalign{\centering $T_{pos}(n,1) = 0, \qquad T_{pos}(1,p)\leq Cp$}
\end{align*}
for a suitable constant $C$. These inequalities imply $T_{neg}(n,p)\! \leq \! C'n\log(1+p)$ and $T_{pos}(n,p)\! \leq\! C'(p-1) \log (1+n)$ for another constant $C'$. Thus the running time of Algorithm $1$, where $p=\sP+1$, $n=\sN$, is $O(\sN \log \sP + \sP \log \sN)$.
\end{proposition}

\begin{proof} The recursive inequalities follow from inspection of Algorithm $1$. As for the ``aggregated'' inequalities, we proceed in both cases by induction.
For the first inequality the base step is trivial for high enough constant $C'$ and for the inductive step we may write
\begin{align*}
T_{neg}(n,p) & \leq Cn + T_{neg}(n/2, p_1) + T_{neg}(n/2, p_2) \\
&\leq Cn + \frac12 C'n \log(1+p_1) + \frac12 C'n \log(1+p_2) \\
&= C'n\left( \frac{C}{C'} + \log \sqrt{(1+p_1)(1+p_2)} \right) \\
&\leq C'n \log(p_1+p_2) = C'n\log(1+p)
\end{align*}
where in the last inequality we used that $$1+(1+p_1)(1+p_2) \leq (p_1+p_2)^2$$
for integers $p_1$, $p_2$ with $p_1+p_2 = p+1 \geq 3$. That makes the last inequality true for sufficiently high $C'$ (not depending on $n$ and $p$).

The proof of the second inequality is an easier variation on the previous technique.
\end{proof} 

\subsection{Lower Bound on Complexity}
\label{subsec:lowerBound}

In order to prove the matching lower bound (among comparison-based algorithms), we intend to use the classical information theoretical argument: There are many possible outputs and from each comparison we receive one bit of information, therefore we need ``many'' comparison to shatter all output options.

\begin{proposition}\label{prop:lower_bound} Let $\Delta$ be a loss function. Then any comparison-based algorithm for Problem ($6$) requires $\Omega(\left|\calN\right| \log \left|\calP \right|)$ operations.
\end{proposition}
\begin{proof}
Since the negative samples are unsorted on the input and the scores are arbitrary, every possible mapping from $\{1, \dots, |\calN|\}$ to $\{1, \dots, \left|\calP\right|+1\}$ may induce the (unique) optimal assignment of interleaving ranks. There are $\left(\left|\calP\right|+1\right)^{\left|\calN\right|}$ possibilities to be distinguished and each comparison has only two possible outcomes.
Therefore we need $\log_2\left(\left(\left|\calP\right|+1\right)^{\left|\calN\right|} \right) \in \Omega(\left|\calN\right| \log \left|\calP\right|)$ operations. \hfill \end{proof}

\section{Remaining proofs}\label{sec:proof}

Throughout the text we omitted several proofs, mostly because they are straightforward generalizations of what already appeared in \cite{yuesigir07} and \cite{chakrabartikdd08}. For the sake of completeness, we present them here.

\medskip

\noindent {{\bf Proof of Observation $1$} (of main text) {\bf :}} Let $\bR$ be any optimal solution. We check that $F({\bf X},{\bf R};{\bf w})$ increases if we swap two samples $x,y \in \cal{P}$ in $\bR$ with $ind(\bx) < ind(\by)$ and $\phi(\bx; {\bf w}) < \phi(\by; {\bf w})$ (it boils down to $ac+bd > ad+bc$ for $a > b \geq 0$ and $c > d \geq 0$). Since similar argument applies for negative samples, we can conclude that $\bR$ already has both negative and positive samples sorted decreasingly. Otherwise, one could perform swaps in $\bR$ that would increase the value of the objective, a contradiction with the optimality of $\bR$.
\hfill\ensuremath{\square}

\medskip

\noindent {{\bf Proof of Observation $2$} (of main text) {\bf :}} Recall that $opt_j$ is the highest rank with maximal value of the corresponding $f_{j'}$. It suffices to prove that for $i_{j+1} = \max \argmax f_{j+1}$ and $i_j = \max \argmax f_j$, we have $i_{j+1} \geq i_j$. Since by Observation \ref{obs:fj} functions $f_j$ inherit property (C2), we can compare the discrete derivatives of $f_j$ and $f_{j+1}$, all left to do is to formalize the discrete analogue of what seems intuitive for continuous functions. 

Assume $i_{j+1} < i_j$. Then since

\begin{align*}
f_{j+1}(i_j) - f_{j+1}(i_{j+1}) &= \sum_{i=i_{j+1}}^{i_j-1} f_{j+1}(i+1) - f_{j+1}(i) \\
&\geq \sum_{i=i_{j+1}}^{i_j-1} f_{j}(i+1) - f_{j}(i)\\
& = f_{j}(i_j) - f_{j}(i_{j+1}) \geq 0,
\end{align*}
we obtain that $i_j \in \argmax f_{j+1}$ and as $i_j > i_{j+1} = \max \argmax f_{j+1}$ and we reached the expected contradiction.
\hfill\ensuremath{\square}

\begin{lemma}\label{objfuncdecomp} The objective function $F({\bf X},{\bf R};{\bf w})$ decomposes into contributions of negative and positive samples as follows:
\begin{align*}
F({\bf X},{\bf R};{\bf w}) &= \frac{1}{\left|\cal P\right| \left| N \right|} \sum_{\bx \in \cal{P}} \sum_{\by \in \cal{N}} {\bf R}_{\bx,\by}(\phi(\bx; {\bf w}) - \phi(\by; {\bf w})) \\
 &= \sum_{\bx \in \cal{P}} c(\bx) \phi(\bx; {\bf w}) + \sum_{\by \in \cal{N}} c(\by) \phi(\by; {\bf w}),
\end{align*}
where
$$c(\bx)\! =\! \frac{\sN+2 \!-\! 2rank(\bx)}{\sP \sN}, \, c(\by) \!=\! \frac{\sP+2 \!-\! 2rank(\by)}{\sP \sN}.$$

In particular, assuming already that $\{s^+_i\}$ is sorted, and that $\bR$ is induced by a vector of interleaving ranks ${\bf r}$, one has
$$F({\bf X},{\bf R};{\bf w}) = \sum_{i=1}^{\sP} c^+_i s^+_i + \sum_{j=1}^{\sN} c^-_j s^*_j,$$
where
$$c^+_i = \frac{\sN+2 - 2r^+_i}{\sP \sN}, \qquad c^-_j = \frac{\sP+2 - 2r_j}{\sP \sN}.$$
Here $r^+_i$ stands for the interleaving rank of the $i$-th positive sample, which can be computed as $r^+_i = 1+ \lvert \{j : r_j \leq i\} \rvert$.
\end{lemma}
\begin{proof} This is straightforward to verify with a short computation.
\end{proof}

\medskip

\noindent {\bf Proof of Observation \ref{obs:obj_nice}:} We slightly modify the decomposition from Lemma \ref{objfuncdecomp} in order to incorporate the array $\{s^+_i\}$:
\begin{align*}
&F({\bf X},{\bf R};{\bf w}) \\
&= \sum_{\by \in \cal{N}} \!\left( \!c(\by)\phi(\by; {\bf w}) +  \sum_{\bx \in \cal{P}} \bR_{\bx,\by} \phi(\bx; {\bf w}) \right) \\
&=  \frac{1}{\sN \sP} \sum_{j=1}^{\sN} \left(\! (\sP\!+\! 2 \!-\! 2r_j)s^*_j \!+\!  2\sum_{i=1}^{r_j\!- \!1} s^+_i \!-\! \sum_{i=1}^{\sP} s^+_i \right).
\end{align*}
This, in combination with (C1), defines the functions $f_j$ for $j = 1, \dots, \sN$.
As for the condition (C2), we have 
$$f_j(i+1) - f_j(i) = \frac{2(s^+_i - s^*_j)}{\sN \sP} + \delta_j(i+1) - \delta_j(i),$$
where, let us be reminded, $\{s^*_j\}$ is the sorted array of scores of negative samples.
After writing analogous equality for $j+1$ and using that (C2) holds for functions $\delta_j$, we can check that the desired inequality
$$f_{j+1}(i+1) - f_{j+1}(i) \geq f_j(i+1) - f_j(i)$$
follows from $s^*_{j+1} \leq s^*_j$.

Note that for computing the $\argmax f_j(i)$ it is sufficient to compute all discrete derivatives (i.e. all the differences $f_j(i+1) - f_j(i)$); the actual values of $f_j$ are in fact not needed. For $\delta_j$ we know that one such evaluation is constant time and for $f_j$ this is also the case since we assumed to have access to $s^*_j$.
\hfill\ensuremath{\square}

\section{NDCG and Discount Functions} \label{sec:NDCG}

Chakrabarti {\em et al. }\cite{chakrabartikdd08} use a slightly modified definition of the discount $D(\cdot)$ as
$$D(i) = \begin{cases}
               1 \qquad &1 \leq i \leq 2 \\
               1/\log_2(i) \qquad &i > 2 \\               
          \end{cases}.$$
For the resulting {\sc ndcg} loss, a greedy algorithm is proposed for solving the loss augmented inference problem.
This algorithm achieves the runtime of $O(\sN \sP + \sN \log \sN)$. The authors also suggest to use a cut-off $k$ in the definition of discount $D(i)$, setting $D(i) = 0$ for $i \geq k$. With this simplification they achieved a reduced complexity of $O( (\sN + \sP) \log( \sP + \sN) + k^2)$.

However, with the above definition of a discount, it is possible
to obtain a corner-case where their proof of correctness of the greedy algorithm is not valid (specifically, there exists a counter-example
for Fact 3.4 of~\cite{chakrabartikdd08}). For the greedy algorithm to be correct, it turns out that the convexity of $D(i)$ is essential. 

\begin{remark} Observation $2$ is not true for $\Delta_{NDCG}$ with function $D(i)$ taken as
$$D(i) = \begin{cases}
               1 \qquad &1 \leq i \leq 2 \\
               1/\log_2(i) \qquad &i > 2 \\               
\end{cases}.$$
\end{remark}
\begin{proof} Consider negative samples $\bx_1$ and $\bx_2$ and a positive sample $\bx_3$ with scores $s_1 = 3\varepsilon$, $s_2 = \varepsilon$, $s_3 = 5\varepsilon$, where $\varepsilon > 0$ is small.

Note that the {\sc ndcg} loss of a ranking $\bR$ reduces to $\Delta_{NDCG}(\bR^*, \bR) = 1 - D(ind(\bx_3))$ where we used the fact that $D(1) = 1$.

The decomposition $\Delta_{NDCG}(\bR^*, \bR) = \delta_1(r_1) + \delta_2(r_2)$ holds if we set
\begin{align*}
\delta_1(1) &= \delta_2(1) = 0, \\ \delta_2(1) &= D(2) - D(3), \\ \delta_1(2) &= D(1) - D(2)=0
\end{align*}

and (possibly by looking at the proof of Observation \ref{obs:fj}) we also find values of $f_1$ and $f_2$ as
\begin{align*}
f_1(1) &= \frac{1}{2}( s_1 - s_3)+\delta_1(1) = -\varepsilon < \varepsilon \\ &= \frac{1}{2}(s_3 - s_1) + \delta_1(2) =  f_1(2) \\
f_2(1) &= \frac{1}{2}( s_2 - s_3)+\delta_2(1) = -2\varepsilon+D(2) - D(3) > 2\varepsilon \\ &= \frac{1}{2}( s_3 - s_2) + \delta_2(2) = f_2(2).
\end{align*} 
Hence $opt_1 = 2 > 1 = opt_2$, a contradiction. 

\end{proof}

\begin{table}[h]
\center{
\begin{tabular}{|l|c|r|}
\hline
Object class & 0-1 loss & {\sc ap} loss\\
\hline
Jumping  & $52.580$ & $55.230$\\
Phoning  & $32.090$ & $32.630$\\
Playing instrument  & $35.210$ & $41.180$\\
Reading  & $27.410$ & $26.600$\\
Riding bike  & $72.240$ & $81.060$\\
Running  & $73.090$ & $76.850$\\
Taking photo  & $21.880$ & $25.980$\\
Using computer  & $30.620$ & $32.050$\\
Walking  & $54.400$ & $57.090$\\
Riding horse  & $79.820$ & $83.290$\\
\hline
\end{tabular}
}
\caption{\emph{Performance of classification models trained by optimizing 0-1 loss and {\sc ap} loss, in terms of {\sc ap} on the test set for the different action classes of {\sc pascal voc} 2011 action dataset.}}
\label{tab:testAP}
\end{table}

\section{Additional Experimental Results}
\label{sec:addexpt} 

For the action classification experiments on the {\sc pascal voc} 2011 data set, we report the performance of models trained by optimizing 0-1 loss as well as {\sc ap} loss in Table~\ref{tab:testAP}. Specifically, we report the {\sc ap} on the test set for each of the 10 action classes. Similarly, we also report  the performance of models trained by optimizing 0-1 loss as well as {\sc ndcg} loss, in terms of {\sc ndcg} on the test set in Table~\ref{tab:testNDCG}. 

For our object detection experiments, we report the detection {\sc ap} in Table~\ref{tab:obj_detect} for all the 20 object categories obtained by models trained using 0-1 loss as well as {\sc ap} loss. For all object categories other than 'bottle', {\sc ap} loss based training does better than that with 0-1 loss. For 15 of the 20 object categories, we get statistically significant improvement with {\sc ap} loss trained models compared to those trained using 0-1 loss (using paired t-test with p-value less than 0.05). While optimizing {\sc ap} loss for learning gives an overall improvement of 7.12\% compared to when using 0-1 loss, for 5 classes it gives an improvement of more than 10\%. The bottom 2 classes with the least improvement obtained by {\sc ap} loss based training, 'chair' and 'bottle' seem to be difficult object categories to detect, with detectors registering very low detection {\sc ap}s. In conjunction with the overall superior performance of {\sc ap} loss for learning model parameters, the efficient method proposed by this paper makes a good case for optimizing {\sc ap} loss rather than 0-1 loss for tasks like object detection.

\begin{table}[h]
\center{
\begin{tabular}{|l|c|r|}
\hline
Object class & 0-1 loss & {\sc ndcg} loss\\
\hline
Jumping  & $86.409$ & $87.895$\\
Phoning  & $73.134$ & $76.733$\\
Playing instrument  & $81.533$ & $83.666$\\
Reading  & $74.528$ & $75.588$\\
Riding bike  & $94.928$ & $95.958$\\
Running  & $93.766$ & $93.776$\\
Taking photo  & $74.058$ & $76.701$\\
Using computer  & $79.518$ & $78.276$\\
Walking  & $89.789$ & $89.742$\\
Riding horse  & $96.160$ & $96.875$\\
\hline
\end{tabular}
}
\caption{\emph{Performance of classification models trained by optimizing 0-1 loss and {\sc ndcg} loss, in terms of {\sc ndcg} on the test set for the different action classes of {\sc pascal voc} 2011 action dataset. We conduct 5-fold cross-validation and report the mean {\sc ndcg} over the five validation sets.}}
\label{tab:testNDCG}
\end{table}

\begin{table}[h]
\center{
\begin{tabular}{|l|c|c|}
\hline
Object category & 0-1 loss & {\sc ap} loss \\ 
\hline\hline

Aeroplane  & $46.60$ & $48.18$ \\
Bicycle  & $48.53$ & $61.45$ \\
Bird  & $33.31$ & $36.73$ \\
Boat  & $15.23$ & $19.66$ \\
Bottle  & $6.10$ & $1.01$ \\
Bus  & $37.01$ & $49.51$ \\
Car  & $61.28$ & $66.78$ \\
Cat  & $38.12$ & $40.77$ \\
Chair  & $2.71$ & $3.23$ \\
Cow  & $21.06$ & $38.52$ \\
Dining-table  & $14.20$ & $39.53$ \\
Dog  & $33.55$ & $36.25$ \\
Horse  & $46.14$ & $53.86$ \\
Motorbike  & $29.97$ & $34.81$ \\
Person  & $29.58$ & $30.41$ \\
Potted-plant  & $21.27$ & $23.03$ \\
Sheep  & $11.65$ & $32.20$ \\
Sofa  & $36.66$ & $42.03$ \\
Train  & $29.71$ & $37.10$ \\
TV-monitor  & $27.31$ & $37.26$ \\

\hline
\end{tabular}
}
\caption{\emph{Performance of detection models trained by optimizing 0-1 loss and {\sc ap} loss, in terms of {\sc ap} on the test set for the different object categories of {\sc pascal voc} 2007 test set.}}
\label{tab:obj_detect}
\end{table}

{\small
\bibliographystyle{ieee}
\bibliography{main}
}


\end{document}